\documentclass[twocolumn, switch]{article} % Method A for two-column formatting

\usepackage{preprint}

\usepackage{amsmath, amsthm, amssymb, amsfonts}
\usepackage{graphicx}
\usepackage{natbib}
\usepackage[utf8]{inputenc}	% allow utf-8 input
\usepackage[T1]{fontenc}	% use 8-bit T1 fonts
\usepackage{xcolor}		% colors for hyperlinks
\usepackage[colorlinks = true,
            linkcolor = purple,
            urlcolor  = blue,
            citecolor = cyan,
            anchorcolor = black]{hyperref}	% Color links to references, figures, etc.
\usepackage{booktabs} 		% professional-quality tables
\usepackage{nicefrac}		% compact symbols for 1/2, etc.
\usepackage{microtype}		% microtypography
\usepackage{lineno}		% Line numbers
\usepackage{float}			% Allows for figures within multicol

\usepackage{lipsum}		%  Filler text

\usepackage{newfloat}
\DeclareFloatingEnvironment[name={Supplementary Figure}]{suppfigure}
\usepackage{sidecap}
\sidecaptionvpos{figure}{c}

\usepackage{titlesec}
\titlespacing\section{0pt}{12pt plus 3pt minus 3pt}{1pt plus 1pt minus 1pt}
\titlespacing\subsection{0pt}{10pt plus 3pt minus 3pt}{1pt plus 1pt minus 1pt}
\titlespacing\subsubsection{0pt}{8pt plus 3pt minus 3pt}{1pt plus 1pt minus 1pt}

\usepackage{tikz,xcolor,hyperref}

\definecolor{lime}{HTML}{A6CE39}
\DeclareRobustCommand{\orcidicon}{
	\begin{tikzpicture}
	\draw[lime, fill=lime] (0,0)
	circle [radius=0.16]
	node[white] {{\fontfamily{qag}\selectfont \tiny ID}};
	\draw[white, fill=white] (-0.0625,0.095)
	circle [radius=0.007];
	\end{tikzpicture}
	\hspace{-2mm}
}
\foreach \x in {A, ..., Z}{\expandafter\xdef\csname orcid\x\endcsname{\noexpand\href{https://orcid.org/\csname orcidauthor\x\endcsname}
			{\noexpand\orcidicon}}
}
\usepackage{graphicx}
\usepackage{amsthm}
\usepackage{dsfont}
\usepackage{placeins}
\usepackage{booktabs}
\usepackage{caption}
\usepackage{subcaption}
\usepackage{placeins}
\newtheorem{theorem}{Theorem}    % theorem
\newtheorem{corollary}[theorem]{Corollary}      % corollary
\newtheorem{proposition}[theorem]{Proposition}   % proposition

\usepackage[table]{xcolor}
\usepackage{algorithm}
\usepackage{algorithmic}
\usepackage{pgfplots}
\pgfplotsset{compat=1.18}
\usetikzlibrary{arrows.meta,decorations.pathreplacing,positioning}

\usepackage{float}
\usepackage{booktabs}
\usepackage{multirow}
\usepackage{caption}
\usepackage{threeparttable} % Pour les notes de bas de tableau
\usepackage{siunitx} % Pour l'alignement des décimales

\title{Low-Complexity Policy Tessellations in Structured Markov Decision Processes}

\usepackage{xwatermark}
\newwatermark[firstpage,color=gray!60,angle=90,scale=0.32, xpos=-4.05in,ypos=0]{\href{https://doi.org/}{\color{gray}{Publication doi}}}
\newwatermark[firstpage,color=gray!60,angle=90,scale=0.32, xpos=3.9in,ypos=0]{\href{https://doi.org/}{\color{gray}{Preprint doi}}}
\newwatermark[firstpage,color=gray!90,angle=0,scale=0.28, xpos=0in,ypos=-5in]{*correspondence: \texttt{email@institution.edu}}

\newcommand{\result}[2]{$#1$ $ \pm\, #2$}
\newcommand{\best}[2]{\textbf{#1} $ \pm\, #2$}

\usepackage{authblk}

\author[1\thanks{\tt{fredy-vale-manuel.pokou@inria.fr}}]{Frédy. Pokou \orcidA{}}
\affil[1]{Inria, University of Lille, CNRS, Centrale Lille Villeneuve-d’Ascq, France}
\begin{document}

\twocolumn[ % Method A for two-column formatting
  \begin{@twocolumnfalse} % Method A for two-column formatting

\maketitle

\begin{abstract}
We study optimal-policy geometry in structured Markov decision processes. While approximate dynamic programming and reinforcement learning typically approximate high-dimensional value functions, we show that optimal policies induce simpler decision tessellations. We propose boundary-based policy approximations that learn policy regions directly. A policy-loss decomposition links performance degradation to action margins and explains why errors concentrate near indifference boundaries. Inventory control and queue admission experiments show lower policy error, smaller value gaps, faster error decay, and stability than reinforcement learning baselines.
\end{abstract}

\vspace{0.35cm}

  \end{@twocolumnfalse} % Method A for two-column formatting
] % Method A for two-column formatting

%\begin{multicols}{2} % Method B for two-column formatting (doesn't play well with line numbers), comment out if using method A

%%%%%%%%%%%%%%%  Main text   %%%%%%%%%%%%%%%
% \linenumbers

\section{Introduction}
\label{sec1}

Markov decision processes (MDPs) are a standard framework for sequential decision-making under uncertainty and constitute the mathematical basis of dynamic programming, approximate dynamic programming, and reinforcement learning \citep{bellman1957dynamic,puterman1994markov,bertsekas2025neuro,powell2007approximate}. A central difficulty in these methods is that they usually require the approximation of a value function or an action-value function over a potentially large state space. Even when the optimal decision rule is simple, the associated value landscape may be numerically rich, which can lead to unnecessary approximation and sampling burdens.

This paper studies a complementary viewpoint. In a finite-action MDP, an optimal decision depends only on comparisons between action values. Hence, for the purpose of decision-making, the relevant object is not the full surface $Q^*(s,a)$, but the partition of the state space into regions where each action is optimal. We refer to this partition as a policy tessellation. Our main premise is that, in many structured MDPs, this tessellation has substantially lower geometric complexity than the value function that induces it.

This perspective is related to classification-based reinforcement learning \citep{lagoudakis2003reinforcement,lazaric2010analysis}, but our focus is different. Rather than proposing a general replacement for value-based methods, we study the geometry of optimal policies in structured dynamic programs and ask whether their decision boundaries can be approximated directly. We consider linear, neural, margin-aware, and local interpolation approximations of policy regions.

The paper makes three contributions. First, it formalizes policy tessellations and introduces computable geometric diagnostics, including normalized boundary length, boundary fragmentation, and action-region components. Second, it establishes a policy-loss decomposition showing that performance degradation is governed by local action losses and, under margin conditions, by errors concentrated near indifference boundaries. Third, numerical experiments on inventory control and queue admission problems show that boundary-based approximations can achieve near-optimal policies with low effective complexity relative to classical value-based reinforcement learning methods.

Overall, the results suggest that, for structured MDPs, approximating the geometry of optimal decisions may be simpler than approximating complete value functions.

\section{Structured MDPs and Policy Tessellations}
\label{sec2}

We consider a discounted MDP
\begin{equation}
\mathcal{M}=(\mathcal{S},\mathcal{A},P,r,\gamma),
\label{eq:mdp}
\end{equation}
where $\mathcal{S}$ is the state space, $\mathcal{A}=\{1,\ldots,K\}$ is a finite action set, $P(\cdot\mid s,a)$ is the transition kernel, $r:\mathcal{S}\times\mathcal{A}\to\mathbb{R}$ is a bounded one-period reward, and $\gamma\in(0,1)$ is the discount factor. For a stationary deterministic policy $\pi:\mathcal{S}\to\mathcal{A}$, its value function is
\begin{equation}
V^\pi(s)
=
\mathbb{E}^{\pi}_s
\left[
\sum_{t=0}^{\infty}\gamma^t r(S_t,\pi(S_t))
\right].
\label{eq:value_policy}
\end{equation}
The optimal value function satisfies
\begin{equation}
V^*(s)
=
\max_{a\in\mathcal{A}}
\left\{
r(s,a)
+
\gamma\int_{\mathcal{S}}V^*(s')P(ds'\mid s,a)
\right\}.
\label{eq:bellman_optimality}
\end{equation}
The optimal action-value function is
\begin{equation}
Q^*(s,a)
=
r(s,a)
+
\gamma\int_{\mathcal{S}}V^*(s')P(ds'\mid s,a),
\label{eq:qstar}
\end{equation}
so that $V^*(s)=\max_{a\in\mathcal{A}}Q^*(s,a)$.

Since ties may occur, we fix throughout a deterministic tie-breaking rule $\tau$. The optimal policy is then the single-valued map
\begin{equation}
\pi^*(s)
=
\tau\left(\arg\max_{a\in\mathcal{A}}Q^*(s,a)\right).
\label{eq:optimal_policy}
\end{equation}
This convention is used only to assign states on indifference sets to one action region; it does not affect optimality.

For any pair of actions $a,b\in\mathcal{A}$, define the pairwise action gap
\begin{equation}
G_{ab}(s)
=
Q^*(s,a)-Q^*(s,b).
\label{eq:pairwise_gap}
\end{equation}
The corresponding indifference set is
\begin{equation}
\Gamma_{ab}
=
\{s\in\mathcal{S}:G_{ab}(s)=0\}.
\label{eq:indifference_set}
\end{equation}
The optimal decision region associated with action $a$ is
\begin{equation}
\mathcal{R}_a
=
\{s\in\mathcal{S}:\pi^*(s)=a\}.
\label{eq:decision_region}
\end{equation}
The collection
\begin{equation}
\mathcal{T}^*
=
\{\mathcal{R}_a:a\in\mathcal{A}\}
\label{eq:policy_tessellation}
\end{equation}
is called the optimal policy tessellation. It forms a partition of $\mathcal{S}$, up to empty regions.

If $\mathcal{S}$ is endowed with a topology and the functions $Q^*(\cdot,a)$ are continuous, then the topological boundary of the decision regions is contained in the union of pairwise indifference sets. In the finite benchmarks considered below, we avoid topological ambiguity and work directly with a discrete boundary.

Let the structured finite state space be embedded in a two-dimensional rectangular grid,
\begin{equation}
\mathcal{S}_h
=
\{0,\ldots,n_1\}\times\{0,\ldots,n_2\}
\subset\mathbb{Z}^2.
\label{eq:grid_state_space}
\end{equation}
Let $E_h$ be the set of undirected nearest-neighbor grid edges,
\begin{equation}
E_h
=
\big\{
\{s,s'\}:s,s'\in\mathcal{S}_h,\ \|s-s'\|_1=1
\big\}.
\label{eq:grid_edges}
\end{equation}
The discrete policy boundary is
\begin{equation}
B_h
=
\big\{
\{s,s'\}\in E_h:\pi^*(s)\neq\pi^*(s')
\big\}.
\label{eq:discrete_boundary}
\end{equation}
We define the normalized boundary length as
\begin{equation}
L_h
=
\frac{|B_h|}{|E_h|}.
\label{eq:boundary_length}
\end{equation}
A small value of $L_h$ indicates that the policy consists of large homogeneous decision regions separated by relatively few switching edges.

We also use the boundary-state fraction
\begin{equation}
\small
F_h
=
\frac{\left|
\left\{
s\in\mathcal{S}_h:
\exists s'\in\mathcal{S}_h,\ 
\{s,s'\}\in E_h,\ 
\pi^*(s)\neq\pi^*(s')
\right\}
\right|}{|\mathcal{S}_h|}
.
\label{eq:boundary_fraction}
\end{equation}
Let $C_h^B$ denote the number of connected components of this boundary-state set under four-neighbor connectivity. For each action $a$, let $C_h(a)$ be the number of connected components of $\mathcal{R}_a$ on the grid. The total number of action-region components is
\begin{equation}
C_h^R
=
\sum_{a\in\mathcal{A}} C_h(a).
\label{eq:region_components}
\end{equation}
Finally, let $K_h$ be the number of boundary states having at least one horizontal and one vertical boundary neighbor. We define the normalized corner index
\begin{equation}
\kappa_h
=
\frac{K_h}{\max\{|B_h|,1\}}.
\label{eq:corner_index}
\end{equation}
The quantities
\begin{equation}
L_h,\quad F_h,\quad C_h^B,\quad C_h^R,\quad \kappa_h
\label{eq:complexity_vector}
\end{equation}
provide observable diagnostics of the geometric complexity of the optimal policy tessellation. They are not substitutes for statistical complexity measures such as VC dimension \citep{vapnik1998statistical}; rather, they measure the realized decision geometry induced by a given structured MDP.

\section{Boundary-Based Policy Approximation}
\label{sec3}

Section~\ref{sec2} represents the optimal policy as a tessellation of the state space. We now describe approximation schemes that learn this tessellation directly. The object of estimation is the decision map $\pi^*$, not the numerical action-value function $Q^*$. This is consistent with classification-based views of reinforcement learning \citep{lagoudakis2003reinforcement,lazaric2010analysis}, but the emphasis here is on the geometry of the induced policy regions.

Let
\begin{equation}
\mathcal{D}_n
=
\{(S_i,Y_i)\}_{i=1}^n
\label{eq:training_data}
\end{equation}
be a training sample with $S_i\in\mathcal{S}_h$ and
\begin{equation}
Y_i=\pi^*(S_i).
\label{eq:labels}
\end{equation}
In the numerical experiments, the labels are computed by exact dynamic programming on the finite benchmark MDPs. A boundary-based approximation is a map
\begin{equation}
\widehat{\pi}_n:\mathcal{S}_h\to\mathcal{A}
\label{eq:approx_policy}
\end{equation}
trained to approximate $\pi^*$ from $\mathcal{D}_n$. It induces approximate decision regions
\begin{equation}
\widehat{\mathcal{R}}_a
=
\{s\in\mathcal{S}_h:\widehat{\pi}_n(s)=a\},
\qquad a\in\mathcal{A}.
\label{eq:approx_regions}
\end{equation}

We consider four approximation schemes. The first is a linear boundary classifier. For an embedded state representation $x(s)\in\mathbb{R}^d$, define scores
\begin{equation}
g_{\theta,a}(s)
=
w_a^\top x(s)+b_a,
\qquad a\in\mathcal{A},
\label{eq:linear_scores}
\end{equation}
where $\theta=\{(w_a,b_a)\}_{a\in\mathcal{A}}$. The induced policy is
\begin{equation}
\widehat{\pi}_{\theta}(s)
=
\tau\left(\arg\max_{a\in\mathcal{A}}g_{\theta,a}(s)\right).
\label{eq:score_policy}
\end{equation}
This model generates polyhedral decision regions and is therefore a parsimonious approximation when the optimal switching geometry is close to linear or monotone.

The second scheme uses a nonlinear score map
\begin{equation}
g_\theta:\mathcal{S}_h\to\mathbb{R}^K,
\label{eq:neural_scores}
\end{equation}
represented by a feedforward neural network. The policy is again given by \eqref{eq:score_policy}. Compared with the linear model, the neural specification can represent curved or disconnected decision regions, at the cost of a larger hypothesis class.

Both score-based approximations are trained by minimizing a multiclass empirical loss. In the experiments we use cross-entropy,
\begin{equation}
\widehat{\mathcal{L}}_n(\theta)
=
\frac{1}{n}
\sum_{i=1}^n
\ell(g_\theta(S_i),Y_i),
\label{eq:empirical_loss}
\end{equation}
where
\begin{equation}
\ell(z,y)
=
-\log
\left(
\frac{\exp(z_y)}
{\sum_{a\in\mathcal{A}}\exp(z_a)}
\right).
\label{eq:cross_entropy}
\end{equation}
Other classification-calibrated losses could be used \citep{bartlett2006convexity}. The essential point is that the loss targets the induced decision regions, not the numerical values of $Q^*(s,a)$.

The third scheme is margin-aware classification. Define the optimal action margin
\begin{equation}
\Delta(s)
=
Q^*(s,\pi^*(s))
-
\max_{a\neq \pi^*(s)}Q^*(s,a).
\label{eq:margin}
\end{equation}
By construction, $\Delta(s)\ge 0$. Small values of $\Delta(s)$ identify states close, in action-value terms, to an indifference boundary. We therefore use the weighted empirical loss
\begin{equation}
\widehat{\mathcal{L}}^{m}_n(\theta)
=
\frac{1}{\sum_{i=1}^n w_i}
\sum_{i=1}^n
w_i\,\ell(g_\theta(S_i),Y_i),
\label{eq:margin_loss}
\end{equation}
with
\begin{equation}
w_i
=
\frac{1}{\Delta(S_i)+\eta},
\qquad \eta>0.
\label{eq:margin_weights}
\end{equation}
The constant $\eta$ prevents singular weights and controls the intensity of the boundary emphasis.

The fourth scheme is local interpolation. For a query state $s$, let $\mathcal{N}_k(s)$ be the set of $k$ nearest sampled states to $s$ in the embedded state metric. The local boundary rule is
\begin{equation}
\widehat{\pi}_{k}(s)
=
\tau\left(
\arg\max_{a\in\mathcal{A}}
\sum_{S_i\in\mathcal{N}_k(s)}
\omega_i(s)\mathbf{1}_{\{Y_i=a\}}
\right),
\label{eq:knn_policy}
\end{equation}
where $\omega_i(s)\ge 0$ and
\begin{equation}
\sum_{S_i\in\mathcal{N}_k(s)}\omega_i(s)=1.
\label{eq:knn_weights}
\end{equation}
In the experiments, inverse-distance weights are used. This estimator provides a flexible local benchmark, but its effective complexity grows with the sample size.

All four schemes produce approximate tessellations of $\mathcal{S}_h$. Their comparison separates the effects of linear boundary structure, nonlinear boundary geometry, local interpolation, and explicit margin weighting.

\section{Structural Properties of Boundary-Based Policies}
\label{sec4}

This section relates errors in approximate tessellations to policy performance. We assume that rewards are uniformly bounded: there exists $R_{\max}<\infty$ such that
\begin{equation}
|r(s,a)|\le R_{\max},
\qquad (s,a)\in\mathcal{S}\times\mathcal{A}.
\label{eq:bounded_rewards}
\end{equation}
It follows that
\begin{equation}
|V^\pi(s)|\le \frac{R_{\max}}{1-\gamma},
\qquad
|Q^*(s,a)|\le \frac{R_{\max}}{1-\gamma}.
\label{eq:value_bound}
\end{equation}

For any stationary deterministic policy $\pi$, define the local optimality loss
\begin{equation}
\ell_\pi(s)
=
V^*(s)-Q^*(s,\pi(s)).
\label{eq:local_loss}
\end{equation}
Then $\ell_\pi(s)\ge 0$, and $\ell_\pi(s)=0$ whenever $\pi(s)$ is optimal at $s$.

\begin{proposition}[Policy-loss decomposition]
\label{prop:loss_decomp}
For any stationary deterministic policy $\pi$ and any initial state $s$,
\begin{equation}
V^*(s)-V^\pi(s)
=
\mathbb{E}_s^\pi
\left[
\sum_{t=0}^{\infty}
\gamma^t
\ell_\pi(S_t)
\right].
\label{eq:loss_decomp_state}
\end{equation}
Consequently, for any initial distribution $\mu$,
\begin{equation}
J(\pi^*)-J(\pi)
=
\frac{1}{1-\gamma}
\mathbb{E}_{S\sim d_\mu^\pi}
[
\ell_\pi(S)
],
\label{eq:loss_decomp_occupancy}
\end{equation}
where
\begin{equation}
J(\pi)=\mathbb{E}_{S_0\sim\mu}[V^\pi(S_0)]
\label{eq:performance}
\end{equation}
and
\begin{equation}
d_\mu^\pi(B)
=
(1-\gamma)
\sum_{t=0}^{\infty}
\gamma^t
\mathbb{P}_\mu^\pi(S_t\in B)
\label{eq:occupancy}
\end{equation}
is the normalized discounted occupancy measure.
\end{proposition}

\begin{proof}
For any state $s$,
\[
V^*(s)
=
Q^*(s,\pi(s))+\ell_\pi(s).
\]
Using the definition of $Q^*$ gives
\[
V^*(s)-V^\pi(s)
=
\ell_\pi(s)
+
\gamma
\mathbb{E}_s^\pi
[
V^*(S_1)-V^\pi(S_1)
].
\]
Iterating this identity for $T$ steps yields
\[
V^*(s)-V^\pi(s)
=
\mathbb{E}_s^\pi
\left[
\sum_{t=0}^{T-1}\gamma^t\ell_\pi(S_t)
\right]
+
\gamma^T
\mathbb{E}_s^\pi
[
V^*(S_T)-V^\pi(S_T)
].
\]
The final term converges to zero as $T\to\infty$ by \eqref{eq:value_bound}. This proves \eqref{eq:loss_decomp_state}. Integrating with respect to $S_0\sim\mu$ and using \eqref{eq:occupancy} proves \eqref{eq:loss_decomp_occupancy}.
\end{proof}

Since
\[
0\le \ell_\pi(s)\le \frac{2R_{\max}}{1-\gamma},
\]
we obtain the policy-error bound
\begin{equation}
J(\pi^*)-J(\pi)
\le
\frac{2R_{\max}}{(1-\gamma)^2}
\mathbb{P}_{S\sim d_\mu^\pi}
\{\pi(S)\neq \pi^*(S)\}.
\label{eq:policy_error_bound}
\end{equation}
Thus, performance loss depends on how often the approximate tessellation assigns a state to a wrong decision region.

Define the $\varepsilon$-margin neighborhood of the indifference boundary as
\begin{equation}
\mathcal{B}_\varepsilon
=
\{s\in\mathcal{S}:\Delta(s)\le \varepsilon\}.
\label{eq:margin_neighborhood}
\end{equation}

\begin{proposition}[Margin localization]
\label{prop:margin_localization}
Suppose that an approximate policy $\widehat{\pi}$ satisfies
\begin{equation}
\widehat{\pi}(s)=\pi^*(s),
\qquad s\notin\mathcal{B}_\varepsilon .
\label{eq:no_errors_outside_boundary}
\end{equation}
Then
\begin{equation}
J(\pi^*)-J(\widehat{\pi})
\le
\frac{2R_{\max}}{(1-\gamma)^2}
d_\mu^{\widehat{\pi}}(\mathcal{B}_\varepsilon).
\label{eq:margin_localization_bound}
\end{equation}
If, in addition, for some $C>0$ and $\alpha>0$,
\begin{equation}
d_\mu^{\widehat{\pi}}(\mathcal{B}_\varepsilon)
\le
C\varepsilon^\alpha,
\label{eq:margin_condition}
\end{equation}
then
\begin{equation}
J(\pi^*)-J(\widehat{\pi})
\le
\frac{2R_{\max}C}{(1-\gamma)^2}
\varepsilon^\alpha .
\label{eq:margin_rate_general}
\end{equation}
\end{proposition}

\begin{proof}
By \eqref{eq:no_errors_outside_boundary},
\[
\mathbf{1}_{\{\widehat{\pi}(s)\neq\pi^*(s)\}}
\le
\mathbf{1}_{\{s\in\mathcal{B}_\varepsilon\}}.
\]
Substituting this inequality into \eqref{eq:policy_error_bound} with $\pi=\widehat{\pi}$ proves \eqref{eq:margin_localization_bound}. The rate \eqref{eq:margin_rate_general} follows immediately from \eqref{eq:margin_condition}.
\end{proof}

A sharper rate is available when wrong decisions incur a loss no larger than the local margin. This is automatic in binary-action MDPs.

\begin{corollary}[Binary-action margin bound]
\label{cor:binary_margin}
Assume $|\mathcal{A}|=2$. Then, for any deterministic policy $\pi$,
\begin{equation}
\ell_\pi(s)
=
\Delta(s)\mathbf{1}_{\{\pi(s)\neq\pi^*(s)\}}.
\label{eq:binary_local_loss}
\end{equation}
Consequently,
\begin{equation}
J(\pi^*)-J(\pi)
=
\frac{1}{1-\gamma}
\mathbb{E}_{S\sim d_\mu^\pi}
\left[
\Delta(S)\mathbf{1}_{\{\pi(S)\neq\pi^*(S)\}}
\right].
\label{eq:binary_loss_bound}
\end{equation}
If $\pi$ makes errors only on $\mathcal{B}_\varepsilon$ and \eqref{eq:margin_condition} holds, then
\begin{equation}
J(\pi^*)-J(\pi)
\le
\frac{C}{1-\gamma}
\varepsilon^{\alpha+1}.
\label{eq:binary_margin_rate}
\end{equation}
\end{corollary}

\begin{proof}
When $|\mathcal{A}|=2$, a wrong action is necessarily the unique nonoptimal action. Therefore the local loss equals the gap between the optimal and nonoptimal action values, which is exactly $\Delta(s)$. The identity \eqref{eq:binary_loss_bound} follows from Proposition~\ref{prop:loss_decomp}. If errors occur only on $\mathcal{B}_\varepsilon$, then $\Delta(S)\le\varepsilon$ on the error set, and \eqref{eq:binary_margin_rate} follows from \eqref{eq:margin_condition}.
\end{proof}

Proposition~\ref{prop:loss_decomp} and Corollary~\ref{cor:binary_margin} explain why boundary-based approximation can be effective. Errors far from indifference boundaries are costly but easier to avoid; errors near the boundary may be more frequent but have smaller local decision loss. This is precisely the mechanism exploited by the margin-aware approximation in Section~\ref{sec3}.

Finally, value approximation and policy approximation have different informational requirements. Value-based methods seek to approximate the numerical function $Q^*(s,a)$, whereas boundary-based methods only require the signs of pairwise differences
\begin{equation}
\operatorname{sign}\{Q^*(s,a)-Q^*(s,b)\},
\qquad a,b\in\mathcal{A}.
\label{eq:sign_information}
\end{equation}
Hence, many action-value functions may induce the same policy tessellation. The geometric diagnostics introduced in Section~\ref{sec2} provide empirical measures of this realized decision complexity.

\begin{table}[t]
\centering
\caption{Abbreviations used for benchmark environments and learning methods.}
\label{tab:notation_labels}
\vskip 0.1in
\resizebox{7.5cm}{!}{
\begin{tabular}{lll}
\toprule
\textbf{Category} & \textbf{Short label} & \textbf{Full name} \\
\midrule
Environment & INV-Mix   & Inventory-Mixed-0.45 \\
Environment & INV-Lin   & Inventory-Linear-0.45 \\
Environment & INV-N0    & Inventory-Mixed-0.00 \\
Environment & INV-NH    & Inventory-Mixed-0.80 \\
Environment & INV-Sin   & Inventory-Sinusoidal-0.45 \\
Environment & INV-Quad  & Inventory-Quadratic-0.45 \\
Environment & INV-Hard  & Inventory-Hard-0.45 \\
Environment & QUE-Mix   & Queue-Mixed-0.25 \\
Environment & QUE-Lin   & Queue-Linear-0.25 \\
Environment & QUE-Burst & Queue-Bursty-0.25 \\
Environment & QUE-N0    & Queue-Mixed-0.00 \\
\midrule
Method & Linear Boundary        & Boundary-Linear \\
Method & Neural Boundary        & Boundary-MLP \\
Method & Margin Boundary        & Boundary-MLP-Margin \\
Method & Local Boundary         & Boundary-kNN \\
Method & Double Q-learning      & Double-Q \\
Method & FQI                    & FQI-ExtraTrees \\
Method & Tabular Q-learning     & Tabular-Q \\
\bottomrule
\end{tabular}
}
\end{table}

\section{Numerical Experiments}
\label{sec5}

This section evaluates whether the policy-tessellation viewpoint developed in Sections~\ref{sec2}-\ref{sec4} is empirically relevant for structured finite MDPs. The experiments are designed around three questions. First, do exact optimal policies in standard operational benchmarks display simple decision geometry? Second, can this geometry be approximated directly with small policy error and small value loss? Third, is the effect stable across benchmark variants and random seeds?

All environment and method labels used in the numerical section are defined in Table~\ref{tab:notation_labels}. Full implementation details are reported in Appendix Tables~\ref{tab:appendix_protocol} and~\ref{tab:appendix_env_specs}. Unless otherwise stated, all reported statistics are computed over 15 independent random seeds. Appendix Figure~\ref{figA1} reports the corresponding seed-level stability for a representative inventory benchmark.

\subsection{Benchmark environments and learning methods}
\label{subsec:benchmarks_methods}

We consider two finite structured MDP families. The first is an inventory-control problem with state $s=(x,z)$, where $x$ denotes inventory and $z$ denotes a demand-regime state. The action is an order quantity from a finite set. The second is a queue-admission problem with state $s=(q,\lambda)$, where $q$ denotes the queue length and $\lambda$ denotes an arrival-regime state. The action is binary and determines whether an arrival is admitted. These models are intentionally low-dimensional but nontrivial: they have structured transition laws, interpretable operational primitives, and switching-type optimal policies.

The inventory family contains seven variants: INV-Mix, INV-Lin, INV-N0, INV-NH, INV-Sin, INV-Quad, and INV-Hard. These variants modify the demand-regime nonlinearity and the observation-noise level. The queue-admission family contains four variants: QUE-Mix, QUE-Lin, QUE-Burst, and QUE-N0. The purpose of these variants is to test whether the observed decision-boundary structure persists beyond a single parameterization.

We compare four boundary-based approximations with three value-based baselines. The boundary-based methods are Linear Boundary, Neural Boundary, Margin Boundary, and Local Boundary. They correspond, respectively, to the linear classifier, neural classifier, margin-weighted classifier, and local interpolation rule introduced in Section~\ref{sec3}. The value-based baselines are Double Q-learning, fitted Q iteration (FQI), and Tabular Q-learning. The comparison is therefore between methods that directly approximate the decision map $\pi^*$ and methods that first approximate action values and then derive a greedy policy.

\subsection{Metrics}
\label{subsec:metrics}

The primary metric is policy error, defined as the fraction of grid states on which the learned policy differs from the exact optimal policy computed by dynamic programming. For a learned policy $\widehat{\pi}$, this is
\begin{equation}
\mathrm{PE}(\widehat{\pi})
=
\frac{1}{|\mathcal{S}_h|}
\sum_{s\in\mathcal{S}_h}
\mathbf{1}_{\{\widehat{\pi}(s)\neq \pi^*(s)\}}.
\label{eq:empirical_policy_error}
\end{equation}
The secondary metric is the value-gap bound associated with the induced decision loss, as motivated by Section~\ref{sec4}. We also report training time and the effective number of parameters or complexity units. These quantities are not meant to imply an identical statistical capacity across all methods; rather, they provide a compact comparison of accuracy, decision loss, and computational cost.

\subsection{Optimal policy tessellations}
\label{subsec:tessellations}

Figure~\ref{fig1} displays the exact optimal policy tessellations across all benchmark environments. Each panel is obtained from dynamic programming and shows the optimal action assigned to each grid state. The main observation is that the optimal policies are organized into a small number of contiguous decision regions separated by relatively simple switching boundaries. This pattern appears in both inventory and queue-admission benchmarks, despite changes in noise level, nonlinear demand structure, and arrival dynamics.

This figure provides the empirical motivation for the paper. The relevant decision object is not the full numerical value surface, but the partition of the state space into action regions. In these structured MDPs, that partition is visually and geometrically simpler than the value function that induces it.

Figure~\ref{fig2} makes this distinction explicit on a representative instance. The left panel shows the full optimal value landscape, whereas the right panel shows the corresponding action-margin geometry. The value function varies over the entire state space, while the decision-relevant information is concentrated around the indifference region. This supports the informational separation emphasized in Section~\ref{sec4}: optimal decision-making requires the signs and margins of action-value differences, not a uniformly accurate reconstruction of $Q^*$.

\subsection{Main performance results}
\label{subsec:main_results}

Table~\ref{tab:main_results} reports the main aggregate results for the inventory and queue-admission families. The table averages performance over the corresponding benchmark variants and reports mean $\pm$ standard deviation over 15 seeds.

For the inventory benchmarks, the best policy error is obtained by Margin Boundary, with error $0.0083\pm0.0032$. Neural Boundary obtains the smallest value gap, $0.008\pm0.006$, and Local Boundary remains competitive with policy error $0.0098\pm0.0029$ while requiring very small training time. Linear Boundary is less accurate than the nonlinear and local boundary methods, but still substantially improves over the value-based baselines.

The value-based baselines show larger policy errors on the same inventory family. FQI obtains policy error $0.2101\pm0.0182$, while Double Q-learning and Tabular Q-learning remain around $0.62$ and $0.61$, respectively. This does not imply that value-based reinforcement learning is generally ineffective; rather, it shows that in these structured finite MDPs, approximating the complete value object is a less direct route to recovering the optimal decision regions.

For the queue-admission benchmarks, the separation is again clear. Local Boundary reaches essentially zero policy error and zero value gap. Margin Boundary and Neural Boundary also produce near-optimal policies, with policy errors $0.0003\pm0.0006$ and $0.0038\pm0.0020$, respectively. In contrast, Double Q-learning, FQI, and Tabular Q-learning have policy errors between $0.1245$ and $0.1318$. Thus, across both benchmark families, the most accurate methods are those that approximate the policy tessellation directly.

\begin{table}[!ht]
\centering
\caption{Main performance comparison across the inventory and queue-admission benchmarks. Policy Error and Value Gap are reported as mean $\pm$ standard deviation over 15 independent runs. Lower values indicate better performance.}
\label{tab:main_results}
\vskip 0.15in
\resizebox{9.cm}{!}{
\begin{tabular}{llcccc}
\toprule
\textbf{Benchmark} & \textbf{Method} & \textbf{Policy Error} $\downarrow$ & \textbf{Value Gap} $\downarrow$ & \textbf{Train Time (s)} $\downarrow$ & \textbf{Params} \\
\midrule

\multirow{7}{*}{\rotatebox{90}{Inventory}} 

& Linear Boundary      
& \result{0.0408}{0.0023} 
& \result{0.099}{0.010} 
& 0.043 
& 15 \\

& Neural Boundary      
& \result{0.0117}{0.0045} 
& \best{0.008}{0.006} 
& 2.787 
& 4,677 \\

& Margin Boundary      
& \best{0.0083}{0.0032} 
& \result{0.010}{0.005} 
& 2.935 
& 4,677 \\

& Local Boundary       
& \result{0.0098}{0.0029} 
& \result{0.015}{0.006} 
& \textbf{0.004} 
& 6,400 \\

\cmidrule{2-6}

& Double Q-learning    
& \result{0.6236}{0.0152} 
& \result{24.933}{1.144} 
& 0.115 
& 8,250 \\

& FQI                  
& \result{0.2101}{0.0182} 
& \result{1.144}{0.169} 
& 1.300 
& 60 \\

& Tabular Q-learning   
& \result{0.6118}{0.0171} 
& \result{24.609}{1.128} 
& 0.132 
& 4,125 \\

\midrule

\multirow{7}{*}{\rotatebox{90}{Queue Adm.}} 

& Linear Boundary      
& \result{0.0148}{0.0008} 
& \result{0.030}{0.002} 
& 0.005 
& 3 \\

& Neural Boundary      
& \result{0.0038}{0.0020} 
& \result{0.004}{0.005} 
& 2.201 
& 4,417 \\

& Margin Boundary      
& \result{0.0003}{0.0006} 
& \result{0.001}{0.002} 
& 3.722 
& 4,417 \\

& Local Boundary       
& \best{0.0000}{0.0001} 
& \best{0.000}{0.000} 
& \textbf{0.004} 
& 6,400 \\

\cmidrule{2-6}

& Double Q-learning    
& \result{0.1245}{0.0032} 
& \result{2.303}{0.214} 
& 0.115 
& 2,604 \\

& FQI                  
& \result{0.1250}{0.0120} 
& \result{2.303}{0.346} 
& 0.825 
& 60 \\

& Tabular Q-learning   
& \result{0.1318}{0.0048} 
& \result{2.928}{0.383} 
& 0.143 
& 1,302 \\
\bottomrule
\end{tabular}
}
\end{table}

\subsection{Margin localization and sample scaling}
\label{subsec:margin_scaling}

Figure~\ref{fig3} investigates where policy errors occur. The horizontal axis is the optimal action margin $\Delta(s)$, and the vertical axis reports policy error conditional on margin bins. Boundary-based methods concentrate their errors near low-margin states. This is precisely the region in which actions are nearly indifferent and where Proposition~\ref{prop:margin_localization} and Corollary~\ref{cor:binary_margin} predict that mistakes should be less damaging. By contrast, the value-based baselines exhibit larger errors over a wider range of margins.

Figure~\ref{fig4} reports policy-error scaling as the number of training samples or Bellman updates increases. Boundary-based methods display faster error decay on the representative inventory benchmark. Local Boundary and Margin Boundary perform well at moderate sample sizes, while Neural Boundary improves steadily as the sample size increases. The value-based baselines do not exhibit comparable improvement in induced policy accuracy over the same range. This supports the view that sample efficiency is improved when approximation effort is concentrated on decision boundaries rather than on the full value landscape.

Finally, Figure~\ref{fig5} reports final policy errors for all methods across the 11 benchmark environments. The heatmap confirms that the low-error behavior of the boundary-based methods is not confined to a single environment. Margin Boundary, Local Boundary, and Neural Boundary remain accurate across both inventory and queue-admission variants. In contrast, the value-based baselines are less robust, particularly on the inventory benchmarks. Together with Appendix Figure~\ref{figA1}, this indicates that the observed performance differences are stable across random seeds and benchmark variants.

\begin{figure}[!ht]
    \centering
    \includegraphics[width=1.05\linewidth]{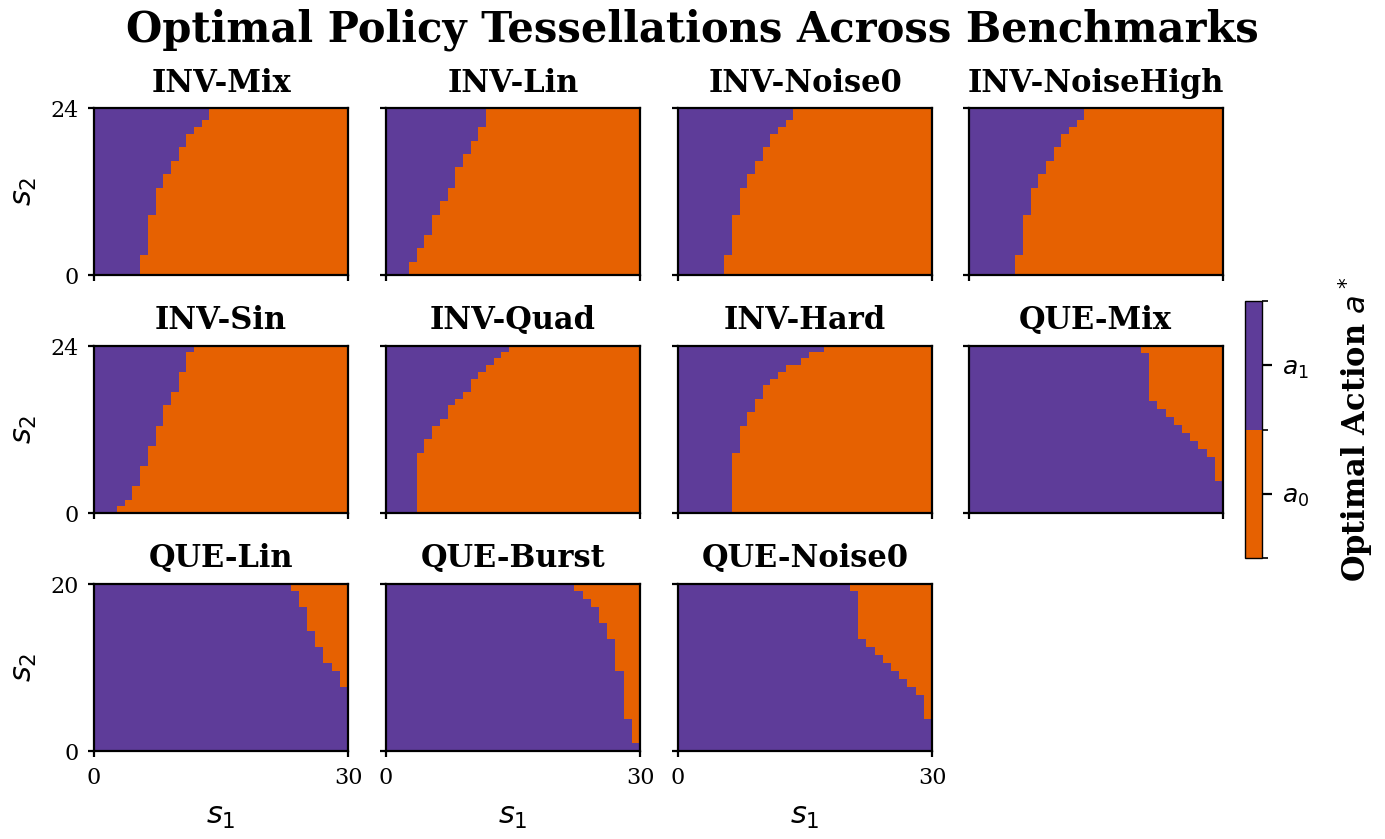}
    \caption{Optimal policy tessellations across inventory and queue-admission benchmarks. Each panel shows the exact optimal decision regions computed by dynamic programming. The figure illustrates that structured MDPs often induce low-complexity policy partitions despite nontrivial value landscapes.}
    \label{fig1}
\end{figure}

\begin{figure}[!ht]
    \centering
    \includegraphics[width=1.02\linewidth]{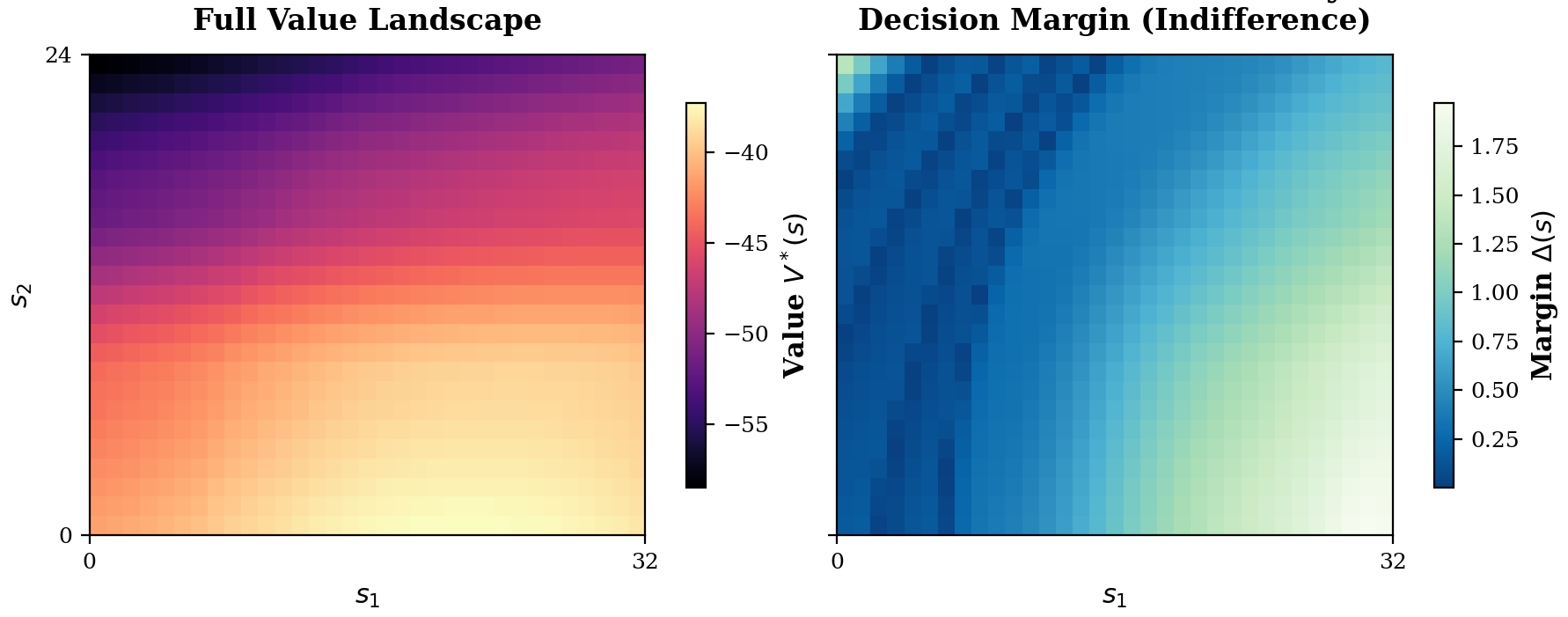}
    \caption{Value learning versus decision-boundary learning. The left panel shows the optimal value landscape, while the right panel shows the corresponding local action-margin geometry. The comparison illustrates that full value approximation contains substantially more numerical information than is needed for optimal action selection.}
    \label{fig2}
\end{figure}

\begin{figure}[!ht]
    \centering
    \includegraphics[width=1.05\linewidth]{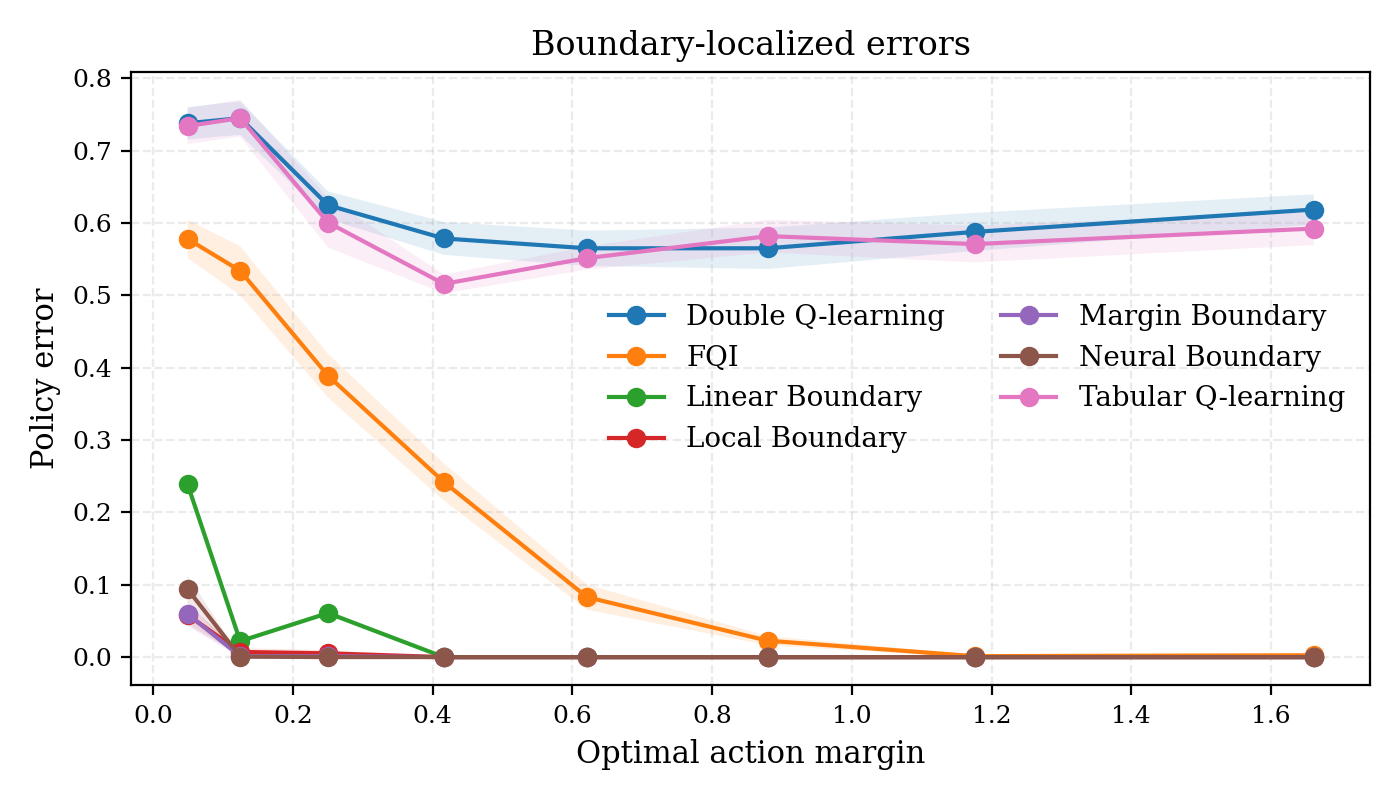}
    \caption{Boundary localization of policy errors. Policy errors are plotted as a function of the optimal action margin. Boundary-based methods concentrate errors near low-margin states, supporting the margin-based loss decomposition in Section~\ref{sec4}.}
    \label{fig3}
\end{figure}

\section{Discussion}
\label{sec6}

The numerical results support the central thesis of the paper: in structured finite MDPs, the geometry of optimal decisions may be substantially simpler than the numerical value function that generates them. Figure~\ref{fig1} shows that the exact optimal policies form simple tessellations across all benchmark variants. Figure~\ref{fig2} further shows that the full value landscape contains information that is not needed for action selection. Table~\ref{tab:main_results} then demonstrates that methods targeting the policy tessellation can achieve small policy errors and small value gaps.

The margin analysis provides the link between the theory and the experiments. Section~\ref{sec4} shows that policy loss depends on the local action loss and, under margin localization, on errors near indifference boundaries. Figure~\ref{fig3} confirms that boundary-based methods make most of their errors in precisely these low-margin regions. This explains why small residual classification errors need not translate into large value losses. Figure~\ref{fig4} further indicates that learning the boundary can yield favorable sample scaling.

The results also clarify the role of model complexity. Neural Boundary and Margin Boundary use richer function classes than Linear Boundary and achieve lower policy error. Local Boundary is highly accurate and fast in these finite grids, but its effective complexity grows with the sample size. Linear Boundary is less flexible, but remains a useful diagnostic: its good performance relative to value-based baselines indicates that much of the relevant policy geometry is already close to low-dimensional switching structure.

The scope of the results should be stated carefully. The benchmarks are finite, structured, and solved exactly to generate reference labels. The paper therefore does not claim that boundary-based approximation universally dominates value-based reinforcement learning. Instead, it identifies a practically relevant setting in which the optimal policy has low realized geometric complexity and can be learned directly. The comparison is consequently about the decision object being approximated: policy regions versus action-value levels.

Several extensions follow naturally. First, adaptive sampling could focus data collection near estimated indifference boundaries. Second, the geometric diagnostics of Section~\ref{sec2} could be used to predict when boundary-based learning is likely to be effective. Third, the approach could be extended to larger continuous-state MDPs by replacing grid-based tessellation measures with continuous geometric or topological proxies.

\begin{figure}[!ht]
    \centering
    \includegraphics[width=1.05\linewidth]{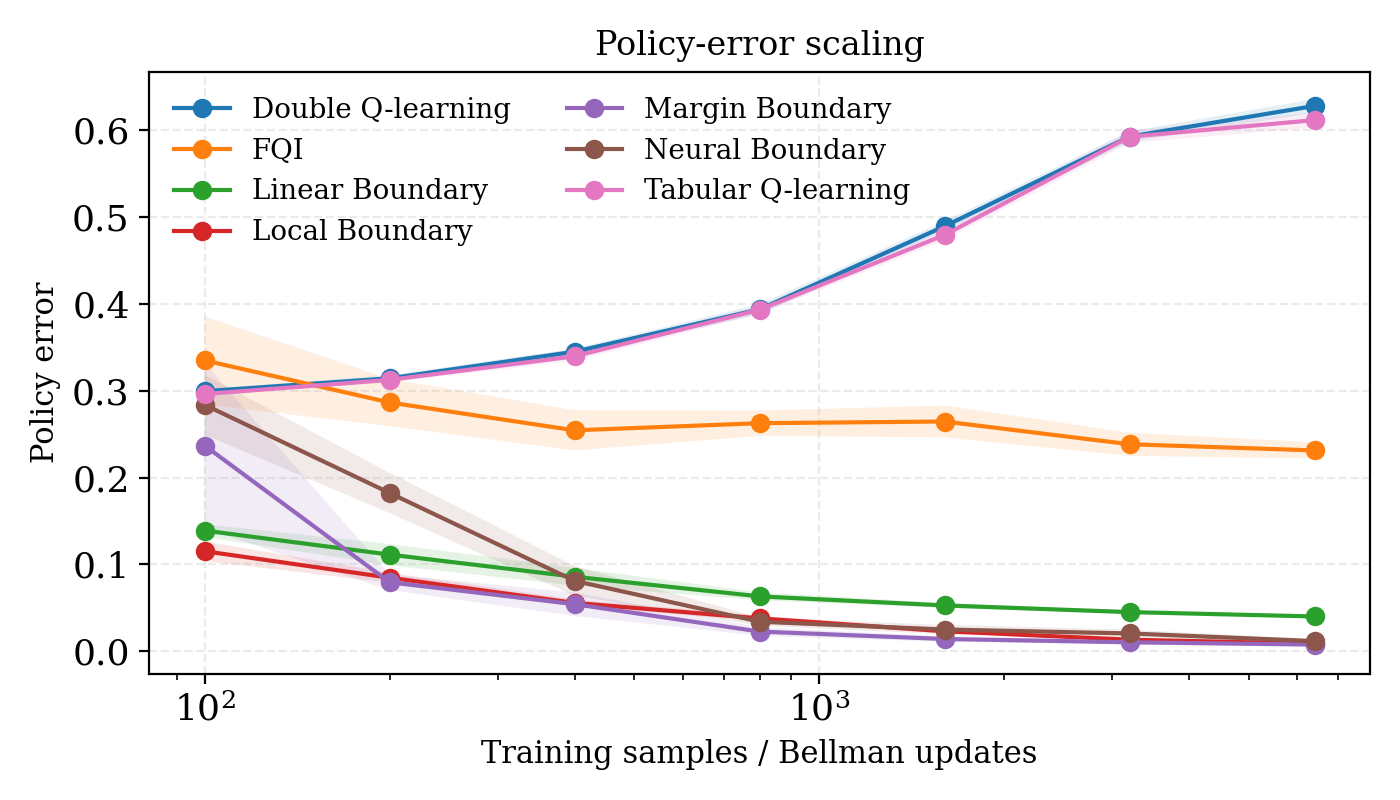}
    \caption{Policy-error scaling with the number of training samples or Bellman updates. Boundary-based methods exhibit faster error decay than value-based baselines, indicating improved sample efficiency in structured MDPs. Shaded regions report variability across 15 independent runs.}
    \label{fig4}
\end{figure}

\begin{figure}[!ht]
    \centering
    \includegraphics[width=1.08\linewidth]{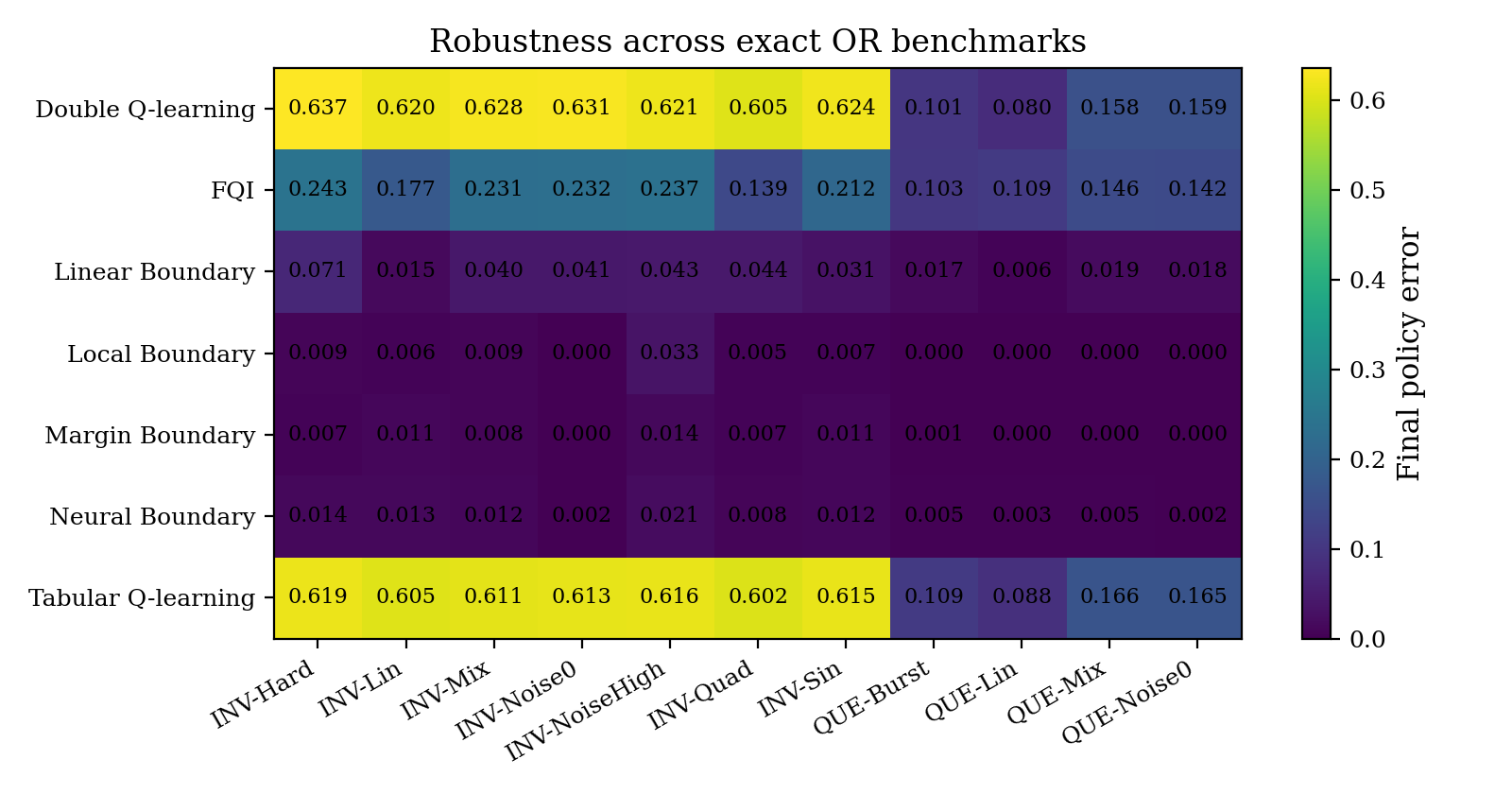}
    \caption{Robustness across exact OR benchmarks. Each cell reports the final policy error for one method and one benchmark environment. Boundary-based approximations remain accurate across both inventory and queue-admission settings, while classical value-based baselines exhibit larger errors.}
    \label{fig5}
\end{figure}

\newpage
\section{Conclusion}
\label{sec7}

This paper proposed policy tessellations as a geometric representation of optimal decision rules in structured MDPs. The main idea is simple: optimal decisions are determined by comparisons between action values, and therefore the decision boundary may be easier to approximate than the full value function.

We formalized this viewpoint, introduced computable diagnostics of policy geometry, and developed boundary-based approximation schemes. The theoretical analysis showed that policy loss is governed by local action losses and by the location of errors relative to indifference boundaries. Numerical experiments on inventory-control and queue-admission benchmarks showed that boundary-based methods achieve near-optimal policies with substantially smaller policy errors and value gaps than standard value-based baselines.

The results suggest that, for structured operational MDPs, learning the geometry of the optimal policy can be a more efficient path to high-quality decisions than learning the complete value landscape. This opens a route toward approximate dynamic programming methods that allocate approximation effort where it matters most: near the boundaries at which optimal decisions change.

\section*{Data Availability}

All numerical experiments in this study are based on synthetic benchmark environments generated algorithmically by the authors. 

\section*{Code Availability}

The Python code used to generate the benchmark environments, compute the optimal policies via dynamic programming, train all boundary-based and reinforcement-learning baselines, and reproduce the tables and figures is available from the corresponding author upon reasonable request.

\FloatBarrier
\setcounter{figure}{0}
\renewcommand{\thefigure}{\Alph{section}\arabic{figure}}

\setcounter{table}{0}
\renewcommand{\thetable}{\Alph{section}\arabic{table}}

\begin{appendix}

\section{Experimental Setup \& Hyperparameters}

\begin{table}[!ht]
\centering
\caption{Common experimental protocol.}
\label{tab:appendix_protocol}
\vskip 0.10in
\resizebox{8.cm}{!}{
\begin{tabular}{ll}
\toprule
\textbf{Parameter} & \textbf{Value} \\
\midrule
Number of independent seeds & 15 \\
Training sample sizes & $\{100,200,400,800,1600,3200,6400\}$ \\
Evaluation policy & Exact optimal policy from dynamic programming \\
Primary metric & Policy error \\
Secondary metric & Value-gap bound \\
Boundary metric & Optimal action margin \\
Reported uncertainty & Mean $\pm$ standard deviation across seeds \\
Benchmark families & Inventory control; queue admission \\
Discount factor & $\gamma=0.95$ \\
\bottomrule
\end{tabular}
}
\end{table}

\begin{table*}[!ht]
\centering
\caption{Environment specifications used in the numerical experiments. All benchmarks use discount factor $\gamma=0.95$. Inventory states are $(x,z)\in\{0,\ldots,x_{\max}\}\times\{0,\ldots,z_{\max}\}$ with order actions in $\mathcal{A}=\{0,4,8,12,16\}$. Queue-admission states are $(q,\lambda)\in\{0,\ldots,q_{\max}\}\times\{0,\ldots,\ell_{\max}\}$ with binary admission actions $\mathcal{A}=\{0,1\}$.}
\label{tab:appendix_env_specs}
\vskip 0.10in
\begin{scriptsize}
\setlength{\tabcolsep}{3.5pt}
\resizebox{18.cm}{!}{
\begin{tabular}{llcccccccc}
\toprule
\textbf{Label} & \textbf{Type} & \textbf{Nonlin.} & $\sigma$ & \textbf{State bounds} & $\mathcal{A}$ & \textbf{Demand/arrival} & \textbf{Costs/rewards} & \textbf{Trunc.} & \textbf{Strength} \\
\midrule

INV-Mix 
& Inventory & mixed & 0.45 
& $x_{\max}=32,\ z_{\max}=24$
& $\{0,4,8,12,16\}$
& $b=1.8,\ \beta_z=0.36$
& $c_o=0.25,\ K=0.20,\ h=0.08,\ p=1.55$
& $d_{\max}=55$
& 1.0 \\

INV-Lin 
& Inventory & linear & 0.45 
& $x_{\max}=32,\ z_{\max}=24$
& $\{0,4,8,12,16\}$
& $b=1.8,\ \beta_z=0.36$
& $c_o=0.25,\ K=0.20,\ h=0.08,\ p=1.55$
& $d_{\max}=55$
& 0 \\

INV-N0 
& Inventory & mixed & 0.00 
& $x_{\max}=32,\ z_{\max}=24$
& $\{0,4,8,12,16\}$
& $b=1.8,\ \beta_z=0.36$
& $c_o=0.25,\ K=0.20,\ h=0.08,\ p=1.55$
& $d_{\max}=55$
& 1.0 \\

INV-NH 
& Inventory & mixed & 0.80 
& $x_{\max}=32,\ z_{\max}=24$
& $\{0,4,8,12,16\}$
& $b=1.8,\ \beta_z=0.36$
& $c_o=0.25,\ K=0.20,\ h=0.08,\ p=1.55$
& $d_{\max}=55$
& 1.0 \\

INV-Sin 
& Inventory & sinusoidal & 0.45 
& $x_{\max}=32,\ z_{\max}=24$
& $\{0,4,8,12,16\}$
& $b=1.8,\ \beta_z=0.36$
& $c_o=0.25,\ K=0.20,\ h=0.08,\ p=1.55$
& $d_{\max}=55$
& 1.0 \\

INV-Quad 
& Inventory & quadratic & 0.45 
& $x_{\max}=32,\ z_{\max}=24$
& $\{0,4,8,12,16\}$
& $b=1.8,\ \beta_z=0.36$
& $c_o=0.25,\ K=0.20,\ h=0.08,\ p=1.55$
& $d_{\max}=55$
& 1.0 \\

INV-Hard 
& Inventory & hard & 0.45 
& $x_{\max}=32,\ z_{\max}=24$
& $\{0,4,8,12,16\}$
& $b=1.8,\ \beta_z=0.36$
& $c_o=0.25,\ K=0.20,\ h=0.08,\ p=1.55$
& $d_{\max}=55$
& 1.2 \\

\midrule

QUE-Mix 
& Queue & mixed & 0.25 
& $q_{\max}=30,\ \ell_{\max}=20$
& $\{0,1\}$
& $\lambda_0=1.0,\ \beta_\lambda=0.35,\ \mu=4.0$
& $r_a=2.0,\ h=0.10,\ c_{\rm ov}=4.0,\ c_{\rm rej}=0.40$
& $d_{\max}=45,\ s_{\max}=20$
& -- \\

QUE-Lin 
& Queue & linear & 0.25 
& $q_{\max}=30,\ \ell_{\max}=20$
& $\{0,1\}$
& $\lambda_0=1.0,\ \beta_\lambda=0.35,\ \mu=4.0$
& $r_a=2.0,\ h=0.10,\ c_{\rm ov}=4.0,\ c_{\rm rej}=0.40$
& $d_{\max}=45,\ s_{\max}=20$
& -- \\

QUE-Burst 
& Queue & bursty & 0.25 
& $q_{\max}=30,\ \ell_{\max}=20$
& $\{0,1\}$
& $\lambda_0=1.0,\ \beta_\lambda=0.35,\ \mu=4.0$
& $r_a=2.0,\ h=0.10,\ c_{\rm ov}=4.0,\ c_{\rm rej}=0.40$
& $d_{\max}=45,\ s_{\max}=20$
& -- \\

QUE-N0 
& Queue & mixed & 0.00 
& $q_{\max}=30,\ \ell_{\max}=20$
& $\{0,1\}$
& $\lambda_0=1.0,\ \beta_\lambda=0.35,\ \mu=4.0$
& $r_a=2.0,\ h=0.10,\ c_{\rm ov}=4.0,\ c_{\rm rej}=0.40$
& $d_{\max}=45,\ s_{\max}=20$
& -- \\

\bottomrule
\end{tabular}
}
\end{scriptsize}

\vspace{1mm}
\begin{scriptsize}
\emph{Notes.} $\sigma$ denotes observation noise. For inventory benchmarks, $b$ is the demand base, $\beta_z$ is the demand-regime slope, $c_o$ is the unit ordering cost, $K$ is the fixed ordering cost, $h$ is the holding cost, and $p$ is the shortage cost. For queue-admission benchmarks, $\lambda_0$ is the arrival base, $\beta_\lambda$ is the arrival slope, $\mu$ is the service rate, $r_a$ is the admission reward, $c_{\rm ov}$ is the overflow cost, and $c_{\rm rej}$ is the rejection cost.
\end{scriptsize}
\end{table*}

\section{Additional Result}

\begin{figure}[!ht]
    \centering
    \includegraphics[width=1.1\linewidth]{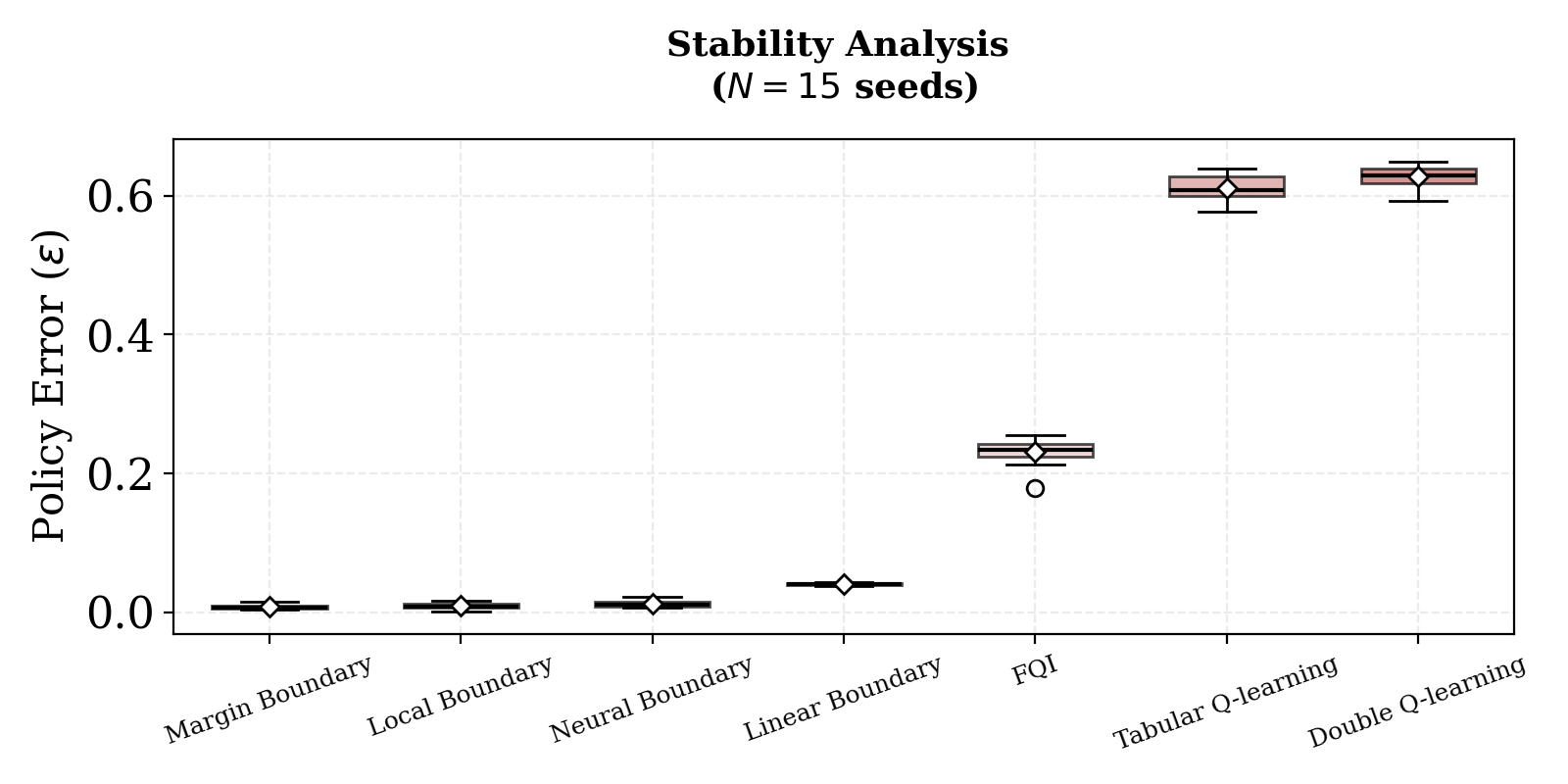}
    \caption{
Stability analysis over 15 random seeds for \texttt{INV-Mix-0.45}. Boundary-based methods show lower final policy error and smaller variability than value-based baselines, indicating greater reproducibility.
}
    \label{figA1}
\end{figure}

\end{appendix}

%%%%%%%%%%%% Supplementary Methods %%%%%%%%%%%%
%\footnotesize
%\section*{Methods}

%%%%%%%%%%%%% Acknowledgements %%%%%%%%%%%%%
%\footnotesize
%\section*{Acknowledgements}

%%%%%%%%%%%%%%   Bibliography   %%%%%%%%%%%%%%
\normalsize
\bibliography{references}

%%%%%%%%%%%%  Supplementary Figures  %%%%%%%%%%%%
%\clearpage

%%%%%%%%%%%%%%%%   End   %%%%%%%%%%%%%%%%
%\end{multicols}  % Method B for two-column formatting (doesn't play well with line numbers), comment out if using method A
\end{document}